%% file: Paper.tex
\documentclass[lettersize,journal]{IEEEtran}

\input{package.tex}

\input{math_commands.tex}

\newcommand{\rev}[1]{{\color{black}#1}}

\begin{document}

\title{Deep Injective Prior for Inverse Scattering}

\author{AmirEhsan Khorashadizadeh$^*$ $^{\orcidlink{0000-0003-0660-6823}}$, Vahid Khorashadizadeh$^*$ $^{\orcidlink{0000-0001-9786-7812}}$,
Sepehr Eskandari$^{\orcidlink{0000-0003-2831-6288}}$,\\
Guy A.E. Vandenbosch$^{\orcidlink{0000-0002-5878-3285}}$,
and Ivan Dokmani\'c$^{\orcidlink{0000-0001-7132-5214}}$
\thanks{
Manuscript received December 29, 2022.\\ 
AmirEhsan Khorashadizadeh and Ivan Dokmani\'c were supported by the European Research Council Starting Grant 852821--SWING.

AmirEhsan Khorashadizadeh is with the Department of Mathematics and Computer Science of the University of Basel, 4001 Basel, Switzerland (e-mail: \href{mailto:amir.kh@unibas.ch}{amir.kh@unibas.ch}).

Vahid Khorashadi-Zadeh is with School of Electric and Computer Engineering, ESAT, Ku-Leuven, Leuven, Belgium.
(e-mail: \href{mailto:vahid.khorashadizadeh@esat.kuleuven.be}{vahid.khorashadizadeh@esat.kuleuven.be}).

Sepehr Eskandari is with Microwave Systems, Sensors, and Imaging Lab (MiXIL), University of Southern California, Los Angeles, CA 90089 USA (e-mail: \href{mailto:sepehres@usc.edu}{sepehres@usc.edu}).

Guy A.E. Vandenbosch is with the ESAT-WaveCoRE Research Division, Department of Electrical Engineering, KU Leuven, 3001 Leuven, Belgium.
(e-mail: \href{mailto:Guy.vandenbosch@kuleuven.be}{Guy.vandenbosch@kuleuven.be}).

Ivan Dokmani\'c is with the Department of Mathematics and Computer Science of the University of Basel, 4001 Basel, Switzerland, and also with the Department of Electrical, Computer Engineering, the University of Illinois at Urbana-Champaign, Urbana, IL 61801 USA (e-mail: \href{mailto: ivan.dokmanic@unibas.ch}{ivan.dokmanic@unibas.ch}).

Our implementation is available at \url{https://github.com/swing-research/scattering_injective_prior}.

{$^*$These authors contributed equally.}
}
}

\maketitle
\begin{abstract}
In electromagnetic inverse scattering, the goal is to reconstruct object permittivity using scattered waves. 
While deep learning has shown promise as an alternative to iterative solvers, it is primarily used in supervised frameworks which are sensitive to distribution drift of the scattered fields, common in practice. Moreover, these methods typically provide a single estimate of the permittivity pattern, which may be inadequate or misleading due to noise and the ill-posedness of the problem.
In this paper, we propose a data-driven framework for inverse scattering based on deep generative models. Our approach learns a low-dimensional manifold as a regularizer for recovering target permittivities. Unlike supervised methods that necessitate both scattered fields and target permittivities, our method only requires the target permittivities for training; it can then be used with any experimental setup. We also introduce a Bayesian framework for approximating the posterior distribution of the target permittivity, enabling multiple estimates and uncertainty quantification.
Extensive experiments with synthetic and experimental data demonstrate that our framework outperforms traditional iterative solvers, particularly for strong scatterers, while achieving comparable reconstruction quality to state-of-the-art supervised learning methods like the U-Net.
\end{abstract}


\section{Introduction}
\IEEEPARstart{e}{lectromagnetic} inverse scattering is the problem of determining the electromagnetic properties of unknown objects from how they scatter incident fields. This non-destructive technique finds applications in various fields, such as early detection of breast cancer~\cite{pardo2021modeling}, mineral prospecting~\cite{dai20233dinvnet}, detecting defects and cracks inside objects~\cite{cao2022pavement}, imaging through the walls~\cite{khorashadi2018through} and remote sensing~\cite{song2023semi}.

While inverse scattering is well-posed and Lipschitz stable in theory, when full-aperture continuous measurements are available~\cite{nachman1996global}, it becomes a severely ill-posed inverse problem for a finite number of measurements. This means that even a small perturbation in the scattered fields can result in a significant error in the reconstructed permittivity pattern~\cite{chen2018computational}. Additionally, the nonlinearity of the forward operator, caused by multiple scattering and amplified by higher permittivity contrasts~\cite{chen2018computational}, further complicates the inversion process. All these together make inverse scattering a challenging problem, especially for strong scatterers (objects with large permittivity) and noisy measurements. To address these challenges, an effective regularization technique is necessary to constrain the search space and achieve accurate recovery.

Several optimization-based methods have been proposed to tackle the nonlinearity and ill-posedness of the inverse scattering problem. These include the Born iterative method~\cite{wang1989iterative}, distorted Born iterative method (DBIM)~\cite{chew1990reconstruction}, contrast source inversion (CSI)~\cite{van1997contrast}, and subspace-based optimization (SOM)~\cite{chen2009subspace}. While these methods have demonstrated effectiveness in reconstructing objects with small permittivity variations, they often fall short in accurately reconstructing objects with large permittivity contrasts. These methods typically rely on iterative optimization of a regularized objective, incorporating manually designed regularization terms~\cite{chen2018computational}.

Deep learning has achieved remarkable success in inverse scattering. Most deep learning models employed for inverse scattering adopt a \textit{supervised} learning approach, which trains a deep neural network to regress the permittivity pattern. Some studies~\cite{khoo2019switchnet, ran2019electromagnetic, khorashadizadeh2022conditional} have utilized scattered fields as the input of the neural network.
Despite the satisfactory reconstructions~\cite{khorashadizadeh2022conditional}, these methods are sensitive to changes in the experimental configuration, such as frequency, the number of transmitters and receivers or other real-world factors. Even slight variations in the distribution of scattered fields in test time can lead to a significant degradation in reconstruction quality, requiring costly acquisition of new training data. Back-projections can be used as input to tackle some of these issues~\cite{li2018deepnis,wei2018deep,fajardo2019phaseless}. While this approach yields good reconstructions for objects with small and moderate permittivity, due to the non-linearity the quality of back-projections significantly drops in large permittivity leading to a drop in the reconstruction quality~\cite{khorashadizadeh2022conditional}. Moreover, supervised learning methods are vulnerable to adversarial attacks~\cite{madry2017towards}, which is problematic in medical applications~\cite{antun2020instabilities}. \rev{Importantly, incorporating the well-established physics of the scattering problem (i.e., the forward operator) to improve the generalization capability is not straightforward in such deep learning models~\cite{chen2020review, fei2022fast, shan2022neural, zhou2022deep, liu2022som, guo2021physics}}.

To tackle these issues, we propose a deep learning approach to inverse scattering using \textit{injective} generative models. The proposed method adopts an unsupervised learning framework---the training phase uses only the target permittivity patterns, and the physics of scattering is fully incorporated into the solution. Deep generative models such as generative adversarial networks (GANs)~\cite{goodfellow2014generative,radford2015unsupervised}, variational autoencoders (VAEs)~\cite{kingma2013auto}, normalizing flows~\cite{dinh2014nice,dinh2016density,kingma2018glow} and diffusion models~\cite{ho2020denoising} belong to a class of unsupervised learning methods and train a deep neural network to transform the samples of a simple (Gaussian) distribution into samples that resemble the target data distribution. Recently, deep generative models (DGM) have been used as a prior for solving inverse problems~\cite{bora2017compressed, ongie2020deep, kothari2021trumpets, kawar2022denoising, vlavsic2022implicit, khorashadizadeh2022funknn}. By leveraging a trained generator on a dataset of target images (the solutions of a given inverse problem), one can explore the latent space of the generator to find a latent code yielding a solution that aligns with the given measurements. 

The choice of generative model is of paramount importance to provide an effective regularization for solving \textit{ill-posed} inverse problems. While GANs have been used as generative priors for inverse problems~\cite{bora2017compressed, kelkar2021prior, karras2019style, hussein2020image}, they are unstable in training~\cite{thanh2020catastrophic,arjovsky2017towards} and result in local minima in iterative approaches~\cite{bora2017compressed}. Normalizing flows resolve some of these issues~\cite{asim2020invertible, whang2020compressed, pmlr-v202-liu23au}, however, they are computationally expensive to train and often do not provide sufficient regularization for highly ill-posed inverse problems. Injective normalizing flows~\cite{brehmer2020flows, kothari2021trumpets, ross2021tractable}, specifically designed for solving ill-posed inverse problems, alleviated these issues; they benefit from a low-dimensional latent space which serves as an effective regularizer for ill-posed inverse problems. \rev{In a related work, Guo et al.~\cite{guo2022nonlinear} employed VAEs as generative priors for inverse scattering.}

In this paper, we use injective flows as generative priors for full-wave inverse scattering. The proposed approach has a significant advantage: it only requires training on the target permittivity patterns and does not require any training data from scattered fields. Once the generator is trained, it can be used to solve inverse scattering problems in arbitrary configurations. This property endows the model with robustness against distribution shifts in the measurements as well as to adversarial attacks. \rev{In contrast to the work of Guo et al.~\cite{guo2022nonlinear}, the invertibility of our generator allows us to perform optimization in both latent and data spaces, providing great flexibility in choosing the scattering solver. Additionally, while Guo et al.~\cite{guo2022nonlinear} require a data-driven initialization, our proposed method can leverage both back-projection and data-driven initializations (among others), making it adaptable to different scenarios and reducing dependence on the particularities of a specific starting point.}
We show that the proposed framework significantly outperforms traditional iterative solvers with reconstructions of comparable or better quality compared to highly-successful supervised methods such as the U-Net~\cite{ronneberger2015u}. 

\rev{All the aforementioned methods reconstruct a single point estimate from the permittivity pattern given the measurements. A point estimate, however, is often insufficient or misleading due to the ill-posedness of the inverse scattering problem. This limitation can be tackled by applying Bayesian frameworks based on deep learning networks~\cite{wei2020uncertainty, khorashadizadeh2022conditional} to generate multiple estimates of the permittivity and perform uncertainty quantification (UQ). 
However, these methods are supervised and suffer from the aforementioned issues. Our second contribution is to leverage our pre-trained injective generator to develop a Bayesian framework that produces multiple estimates of the permittivity pattern enabling the uncertainty quantification. Crucially, the proposed method does not rely on scattered fields during training. As we will discuss in Section~\ref{sec: posterior modeling}, this framework requires injectivity and is thus not practicable with non-injective generators like GANs or VAEs.

This paper is organized as follows. Section~\ref{sec: Problem Statement} provides a brief review of the forward and inverse scattering problem. In Section~\ref{sec: normalizing flows}, we present an overview of normalizing flows and injective flows. Our proposed methods for MAP estimation and posterior modeling in inverse scattering are introduced in Sections~\ref{sec: MAP} and~\ref{sec: posterior modeling}. Computational experiments are presented in Section~\ref{sec: Experiments}. Section~\ref{sec: conclusion} discusses the limitations of our approach and provides insights into future work.
}

\section{Forward and Inverse Scattering}
\label{sec: Problem Statement}
\begin{figure}
    \centering
    \includegraphics[width = 0.45\textwidth]{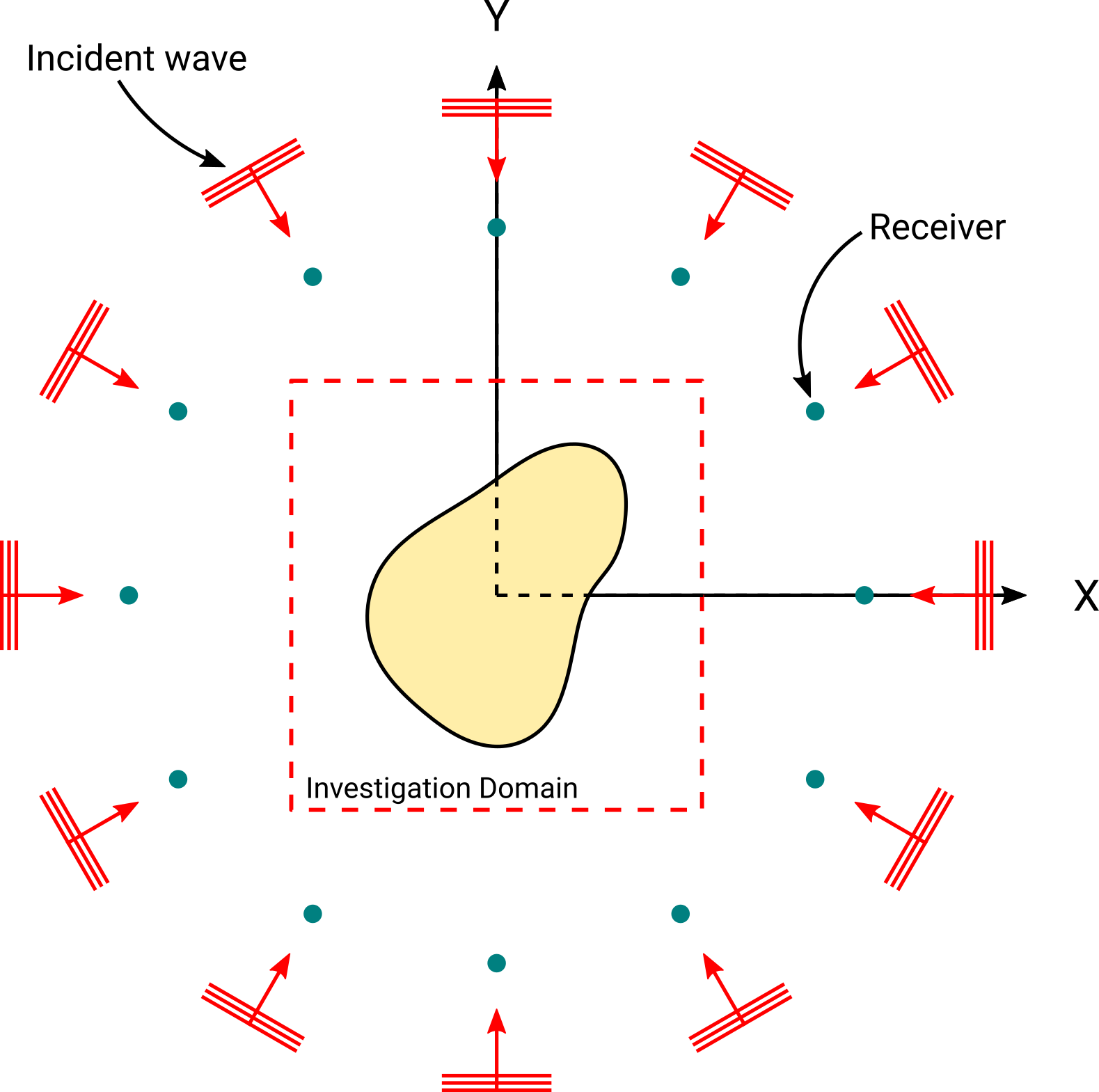}
    \caption{The setup for the inverse scattering problem, red arrows show the incident plane waves; the green circles are the receivers.}
    \label{fig: sensing_geo}
\end{figure}

We begin our discussion with equations governing the 2D forward and inverse scattering problem. We focus on the 2D transverse magnetic ($\text{TM}_z$) case, where the longitudinal direction is along $\hat{z}$. As depicted in Figure~\ref{fig: sensing_geo}, we consider non-magnetic scatterers with permittivity $\epsilon_r$ situated in the investigation domain $\mathcal{D}_{\text{inv}}$, which is a $D \times D$ square. The scatterers are surrounded by a vacuum background with permittivity $\epsilon_0$ and permeability $\mu_0$. The scatterers are illuminated by $N_i$ plane waves with equispaced directions, and $N_r$ receivers are uniformly positioned on a circle with radius $R$ to measure the scattered fields. The forward scattering problem can be derived from the time-harmonic formulation of Maxwell's equations and can be expressed as follows\cite{chen2017spectral},
\begin{equation}
    \nabla \times ( \nabla \times E ^t(r)) - k_0 ^2 \epsilon_r (r) E ^t(r) = i \omega \mu_0 J(r),
    \label{eq: time-harmonic}
\end{equation}
where $E ^t$ represents the total electric field \rev{which has only the $E_z$ component in the $\text{TM}_z$ case}. In addition, $k_0 = \omega \sqrt{\mu_0 \epsilon_0}$ denotes the wavenumber of the homogeneous background, and $J$ corresponds to the contrast current density. The contrast current density, calculated using the equivalence theorem~\cite{rengarajan2000field}, is given by $J(r) = \chi(r) E ^t(r)$, where $\chi(r) = \epsilon_r(r) - 1$ and is referred to as the contrast.  Throughout this paper, the time-dependence factor $\exp(i \omega t)$ with angular working frequency $\omega$ is assumed and will be suppressed for simplicity.

We discretize the investigation domain $\mathcal{D}_{\text{inv}}$ into $N \times N$ units. The state equation can be expressed as,
\begin{equation}
E^t = E^i + G_d \chi E^t, \label{eq: State_Equation}
\end{equation}
where $G_d \in \mathbb{R}^{N^2 \times N^2}$ and $E^{t}$, $E^{i}$ are the total and incident electric fields, respectively; $\chi$ is a diagonal matrix with elements $\chi(n,n) = \epsilon_r(n) - 1$ accounting for the contrast in the medium. On the other hand, the data equation is given by,
\begin{equation}
E^s = G_s \chi E^t + \delta, \label{eq: Data_Equation}
\end{equation}
where $G_s \in \mathbb{R}^{N_r \times N^2}$, $E^s$ denotes the scattered electric fields, and $\delta$ is the additive noise in the measurements. It is worth mentioning that $G_d$ and $G_s$ have closed-form analytical expressions~\cite{chen2018computational}.

We combine~\eqref{eq: State_Equation} and~\eqref{eq: Data_Equation} to obtain a unified expression for the forward model~\cite{chen2018computational},
\begin{equation}
    E^s=G_s \chi(I-G_d \chi)^{-1} E^i + \delta,
\end{equation}
which represents a nonlinear mapping from $\chi$ to $E^s$. For convenience, we define a forward operator $A$ that maps $\chi$ to $E^s$, 
\begin{equation}
    y = A(x) + \delta,
    \label{eq: forward model}
\end{equation}
where $A(\cdot)$ corresponds to the nonlinear forward scattering operator,
\begin{equation}
A(\chi) = G_s \chi {(I-G_d \chi)}^{-1} E^i,
\label{eq: forward operator}
\end{equation}
with $y = E^s$ and $x =\chi$. 
The objective of inverse scattering is to reconstruct the contrast $\chi$ from the scattered fields $E^s$, assuming that $G_d$, $G_s$, incident electric waves $E^i$, and hence the forward operator $A(\cdot)$ are known. In the following section, we will provide a brief overview of deep generative models, focusing specifically on normalizing flows as prior models for inverse problems.

\rev{
\section{Normalizing Flows}
\label{sec: normalizing flows}
Normalizing flows were introduced by Rezende and Mohamed~\cite{rezende2015variational} in the context of variational inference, and by Dinh et al.~\cite{dinh2016density} for density estimation. A normalizing flow $f_\theta$ is an invertible deep neural network, parameterized by a vector of neuron weights $\theta$, that transforms a simple base distribution, typically a Gaussian, $p_Z$, into the target data distribution $p_X$, or an approximation thereof. By transforming a data sample $x$ back to the latent space $z = f_\theta^{-1}(x)$, the likelihood of $x$ can be evaluated as
\begin{equation}
    \log p_X(x) = \log p_Z(f_\theta^{-1}(x)) -  \log | \det J_{f_\theta} |,
    \label{eq: density estimation}
\end{equation}
where $p_Z = \calN(0,I)$ and $J_{f_\theta}$ represents the Jacobian matrix of the neural network $f_\theta$ evaluated at $f_\theta^{-1}(x)$.

Numerous studies have focused on designing invertible neural networks that admit a computationally efficient inverse $f_\theta^{-1}$ and $\log \det$ Jacobian. A staple design block that enables these efficient computations is the so-called coupling layers, introduced by Dinh et al.\cite{dinh2014nice} and further developed in \cite{dinh2016density}. The fact that unlike many other generative models normalizing flows allow for efficient likelihood computation as in~\eqref{eq: density estimation} enables training based on maximum likelihood (ML),
\begin{equation}
    \theta^* = \argmax_\theta \log p_Z(f_\theta^{-1}(x)) -  \log | \det J_{f_\theta} |.
\end{equation}
Normalizing flows also have important limitations. They require bijective neural networks with the same data space dimension throughout the model, resulting in large networks and slow training. Furthermore, as the range of the bijective network is unconstrained and covers the entire space, they do not inherently provide strong regularization for solving ill-posed inverse problems.  In the following section, we will provide a brief review of injective normalizing flows\cite{kothari2021trumpets}, specifically designed for solving ill-posed inverse problems.
}

\subsection{Injective Normalizing Flows}
\label{sec: Trumpets}
\begin{figure*}
    \centering
    \includegraphics[width =0.95\textwidth]{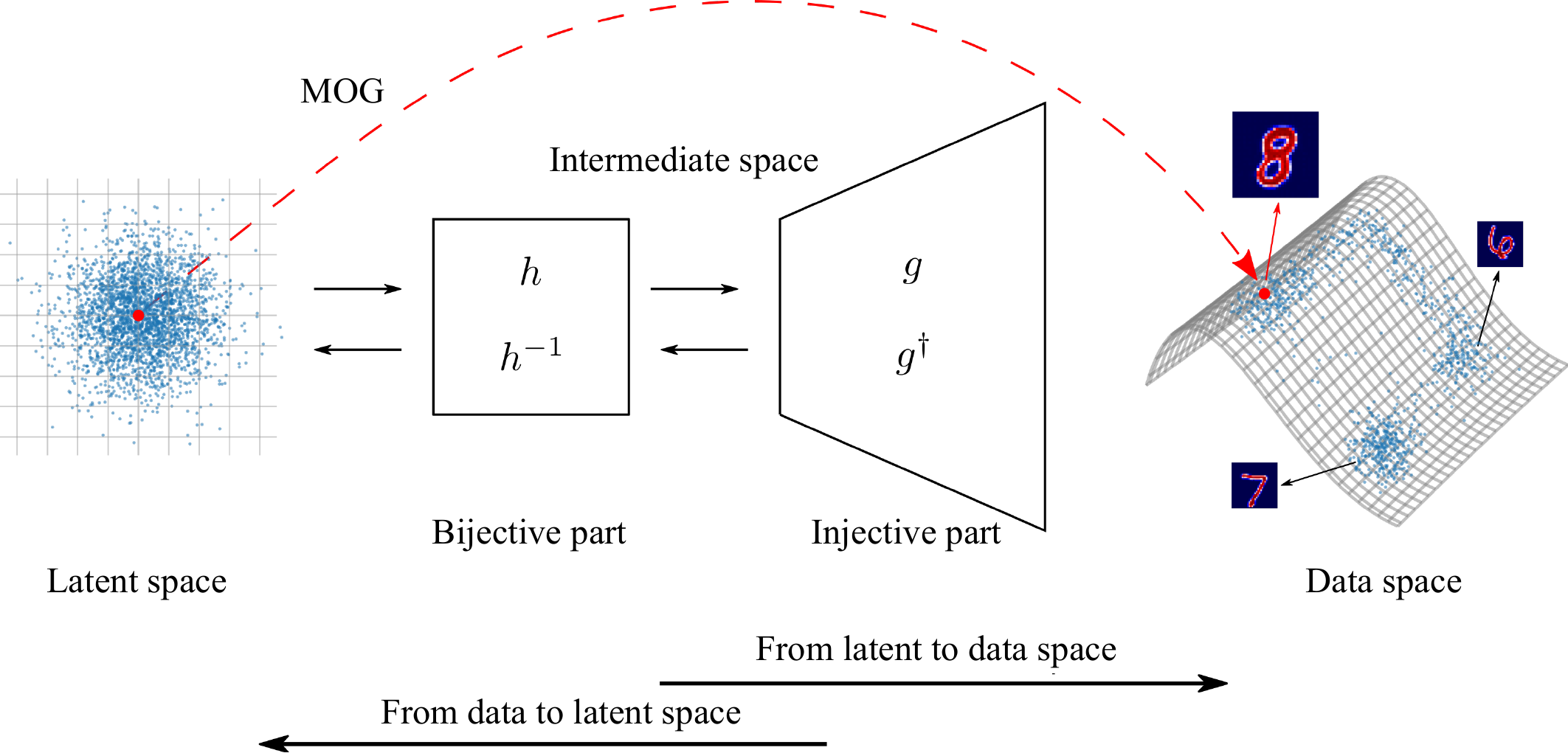}
    \caption{Injective normalizing flows~\cite{kothari2021trumpets} comprise two submodules, a low-dimensional bijective flow $h_\eta$ and an injective network $g_\gamma$ with expansive layers. The MOG initialization, $z = 0$, is illustrated with a red circle in the latent space.}
    \label{fig: Trumpets}
\end{figure*}
\rev{While regular normalizing flows have the same dimension in the latent and data space,}
injective normalizing flows~\cite{brehmer2020flows,kothari2021trumpets, ross2021tractable} map a low-dimensional latent space to the high-dimensional data space using a set of invertible layers. \rev{Injective flows retain the advantages of regular normalizing flows,} including fast inverses and training based on maximum likelihood. As shown in Figure~\ref{fig: Trumpets}, an injective network $f_\theta(z)=g_{\gamma}(h_{\eta}(z))$ with weights $\theta = (\gamma, \eta)$, called a Trumpet, comprises two subnetworks: a bijective part $h_{\eta}$ that maps $\R^d$ to $\R^d$ and an injective part (with expansive layers) $g_{\gamma}$ that maps $\R^d$ to $\R^D$ where $d \ll D$. \rev{Both the bijective and injective subnetworks are composed of revnet blocks. A bijective (injective) revnet block comprises three components: 1) activation normalization, 2) bijective (injective) $1 \times 1$ convolution, and 3) coupling layers:
\begin{enumerate}
    \item Activation normalization,
    \begin{equation}
        \begin{array}{ll}
        \text{\textsc{Forward:}} &  ~ x = \dfrac{z-\mu}{\sigma}  \\
            \text{\textsc{Inverse:}} & ~ z = \sigma x + \mu
        \end{array}
    \end{equation}
    \item $1\times1$ convolution with a kernel $w$,
    \begin{enumerate}
        \item Bijective version:
        \begin{equation}
            \begin{array}{ll}
                    \text{\textsc{Forward:}} & ~ x = w\ast z \\  
                    \text{\textsc{Inverse:}} & ~ z = w^{-1}\ast x
            \end{array}
         \end{equation}
        \item Injective version:

        \begin{equation}
            \begin{array}{ll}
                    \text{\textsc{Forward:}} & ~ x = w\ast z \\  
                    \text{\textsc{Inverse:}} & ~ z = w^{\dagger}\ast x
            \end{array}
         \end{equation}
    \end{enumerate}
     where $w \in \R^{c_{in} \times c_{out}}$ is a $1 \times 1$ convolutional filter, which is simply a matrix multiplication along the channel dimension and $w^\dagger$ is the pseudo-inverse of $w$ (a non-square matrix in the injective dimension-expanding case).
    \item Affine coupling layer
    \begin{equation*}
    \begin{array}{lll}
        \text{\textsc{Forward:}} & x_1 = z_1, \quad x_2 = s(z_1) \circ z_2 + b(z_1)\\
        \text{\textsc{Inverse:}} &z_1 = x_1  , \quad z_2 = s(x_1)^{-1} \circ (x_2 - b(x_1))
    \end{array}
    \end{equation*}
    where $z = [z_1, z_2]^T$ and $x = [x_1, x_2]^2$. The mappings $s$ and $b$ are respectively the scale and the shift networks.
\end{enumerate}
For additional details about the network architecture, please refer to Section~\ref{sec: architecture and training} in the appendix.}

The training process for injective normalizing flows involves two phases, as initially proposed in~\cite{brehmer2020flows}. In the first phase, we adjust the range of the injective generator by optimizing the weights of the injective subnetwork $g_{\gamma}$ to align with the training data,
\rev{
\begin{equation} \label{eq:lmse_conditional}
    \lmse(\gamma) = \\
    \dfrac{1}{N} \sum_{i=1}^N \| x^{(i)}- g_{\gamma}(g_{\gamma}^\dagger(x^{(i)}))\|_2^2,
\end{equation}
where $\{(x^{(i)}, y^{(i)})\}_{i = 1}^{N}$ represents the training data and $g^\dagger$ denotes the layer-wise inverse of the injective subnetwork.

Once the injective subnetwork has been trained for a fixed number of epochs, we move to the second phase where we train the bijective subnetwork $h_{\eta}$ by maximizing the likelihood of the projected training samples in the intermediate space (as shown in Figure~\ref{fig: Trumpets}), 
\begin{equation} \label{eq:LML_conditional}
    \lml(\eta) =
    \dfrac{1}{N}\sum_{i=1}^N \left(-\log p_Z(z^{(i)}) + \log |\det J_{h_{\eta}}|\right),
\end{equation}
where $z^{(i)} = h_{\eta}^{-1}(g_{\gamma}^\dagger(x^{(i)}))$ and $p_Z = \calN(0,I)$.
} Upon completion of training, we can generate random samples similar to the training data using $x_{\text{gen}} = f(z_{\text{gen}})$, where $z_{\text{gen}} \sim \mathcal{N}(0, I)$. Further investigation on the universality of density and manifold approximation of injective flows can be found in~\cite{puthawala2021universal}.

Injective flows, due to their low-dimensional latent space, parameterize a low-dimensional manifold embedded in the high-dimensional data space. During training, this manifold captures plausible samples, making it an effective regularizer for ill-posed inverse problems. \rev{The injective part provides a projection operator on the range of $g_{\gamma}$ as $P_{g_{\gamma}}(x):= g_{\gamma}(g_{\gamma}^\dagger(x))$ which maps the data samples $x$ to the intermediate space by $z^\prime = g_\gamma^\dagger(x)$  and projects them back to the data space by $g_\gamma(z^\prime)$.} Kothari et al. \cite{kothari2021trumpets} employed this projection operator to project a sample onto the manifold in iterative reconstruction schemes. In the next section, we introduce our methodology for solving inverse scattering problems using injective normalizing flows.

\section{MAP inference with Injective Flows for Inverse Scattering}
\label{sec: MAP}

Inverse scattering with partial data is a severely ill-posed inverse problem, which means that a small perturbation in the measurements of scattered fields can result in a significant error in the recovered contrast~\cite{chen2018computational}. As discussed in Section~\ref{sec: Problem Statement}, inverse scattering is a nonlinear inverse problem, with the degree of nonlinearity being strongly influenced by the maximum contrast value. Particularly for objects with large contrasts, the problem becomes highly nonlinear, further increasing the difficulty of the inversion. In such cases, the presence of a robust regularizer that effectively constrains the search space becomes crucially important.

We model the contrast $\chi = x \in \mathcal{X}$ and the scattered fields $E^s = y \in \mathcal{Y}$ as random vectors. For simplicity, we assume that the additive noise $\delta$ in~\eqref{eq: forward model} is a random vector with Gaussian distribution $\delta \sim \calN(0,{\sigma}^2 I)$ although our framework admits other distributions. With this assumption, the likelihood $p_{Y|X}$ can be expressed as,
\begin{equation}
    p_{Y|X}  = \calN(A(X),{\sigma}^2 I).
    \label{eq: noise_distribution}
\end{equation}
An effective approach for solving ill-posed inverse problems is to compute the maximum a posteriori (MAP) estimate, where we seek the solution $x$ that has the highest posterior likelihood given a measurement $y$,
\begin{equation}
    x_{\text{MAP}}  =  \argmax_x ~ \log p_{X|Y}(x|y),
\end{equation}
where $p_{X|Y}(x|y)$ denotes the posterior distribution, representing the conditional distribution of the image of interest given the measurements $y$. \rev{The posterior distribution $p_{X|Y}$ can be computed using Bayes theorem as,
\begin{equation}
    p_{X|Y}(x|y) = \dfrac{p_{Y|X}(y|x) p_X(x)}{\int_x p_{X,Y}(x,y)  dx},
    \label{eq: Bayes}
\end{equation}
which leads to the following expression for the MAP estimate,}
\begin{equation}
    x_{\text{MAP}} = 
    \argmin_x ~ -\log p_{Y|X}(y|x)  - \log p_X(x).
\end{equation}
From~\eqref{eq: noise_distribution} we get
\begin{equation}
    x_{\text{MAP}} =
    \argmin_x ~ \dfrac{1}{2}\|y - A(x) \|_2^2 - \lambda \log p_X(x) ,
    \label{eq:MAP}
\end{equation}
where the first term represents the data-consistency loss while $p_{X}(x)$ denotes the prior distribution of the contrast and yields a regularization term. We additionally insert $\lambda$ as a hyperparameter to adjust the weight of the regularization term as its value depends on the unknown noise power. In general, estimating the prior distribution $p_X$ is challenging, and a commonly used approximation is a Gaussian distribution with zero mean, leading to Tikhonov regularization. However, a Gaussian distribution often deviates significantly from the true prior, resulting in poor reconstructions.

This paper explores a data-driven regularization in inverse scattering based on deep generative models. We leverage a training set of contrast patterns $\{x^{(i)}\}_{i = 1}^N$ and train a deep generative model $x = f(z)$ to produce samples from (approximately) the same distribution as that of the training set.
By sampling from a Gaussian distribution in the latent space $z \in \mathcal{Z}$, we expect the trained generator $f$ to produce plausible contrast samples. This property of deep generative models makes them an effective regularizer for solving inverse problems~\cite{bora2017compressed, guo2022nonlinear}. 

In this paper, we employ injective flows as a generative prior due to their suitability for addressing ill-posed inverse problems~\cite{kothari2021trumpets}. We perform optimization in the latent space to find the latent code which produces a permittivity pattern compatible with the measurements $y$. The optimization problem can be formulated as follows,
\begin{equation}
    z_{\text{MAP}} =
    \argmin_z ~ \dfrac{1}{2}\|y - A(f(z)) \|_2^2 - \lambda \log p_{X}(f(z)),
    \label{eq: LSO}
\end{equation}
\rev{where the regularization term $\log p_{X}$ is approximated via \eqref{eq: density estimation}.} The reconstructed contrast is then obtained as $x_{\text{MAP}} = f(z_{\text{MAP}})$. We call this method latent space optimization (LSO). We note that~\eqref{eq: LSO} has been previously proposed by \cite{asim2020invertible,whang2020compressed} for solving compressed sensing inverse problems using regular normalizing flows.

Unlike the supervised learning methods for inverse scattering~\cite{li2018deepnis, wei2018deep,fajardo2019phaseless,ran2019electromagnetic}, which rely on paired training sets of contrast and scattered fields $\{(x^{(i)}, y^{(i)})\}_{i = 1}^{N}$, our framework is unsupervised, without the need for scattered fields during training. This eliminates the need to retrain the model when the distribution of scattered fields changes due to variations in the experimental configuration. Once the injective generator is trained on the contrast samples, we can directly optimize~\eqref{eq: LSO} for new measurements to reconstruct the corresponding contrast. In addition, our proposed method fully leverages the underlying physics of the scattering problem by optimizing over the complex-valued scattered fields in~\eqref{eq: LSO}. Kothari et al.~\cite{kothari2020learning} have demonstrated that incorporating wave physics into the neural network architecture can significantly enhance the quality of reconstructions, particularly for out-of-distribution data.

\rev{Invertibility of the injective generator allows us} to use an alternative method for~\eqref{eq: LSO} proposed by~\cite{kothari2021trumpets} for linear inverse problems. This method performs the optimization directly in the data space. We call this method data space optimization (DSO) and formulate it as follows,
\begin{equation}
    x_{\text{MAP}} =
    \argmin_x ~ \dfrac{1}{2}\|y - A(g(g^\dagger(x))) \|_2^2 -\lambda \log p_{X}(x)
    \label{eq: DSO}
\end{equation}
where $g(g^\dagger(x))$ represents the projection operator described in section~\ref{sec: Trumpets}.  \rev{Similar to LSO, the second term $\log p_{X}$ can be approximated using~\eqref{eq: density estimation} and acts as an additional regularizer.} In LSO the reconstructed point $x = f(z)$ always lies on the learned manifold; this is not the case for the DSO method, where the reconstructed image may deviate from the manifold. On the other hand, as we discuss next, DSO offers more flexibility in the choice of the initial guess.

\begin{figure}
    \centering
    \includegraphics[width =0.5\textwidth]{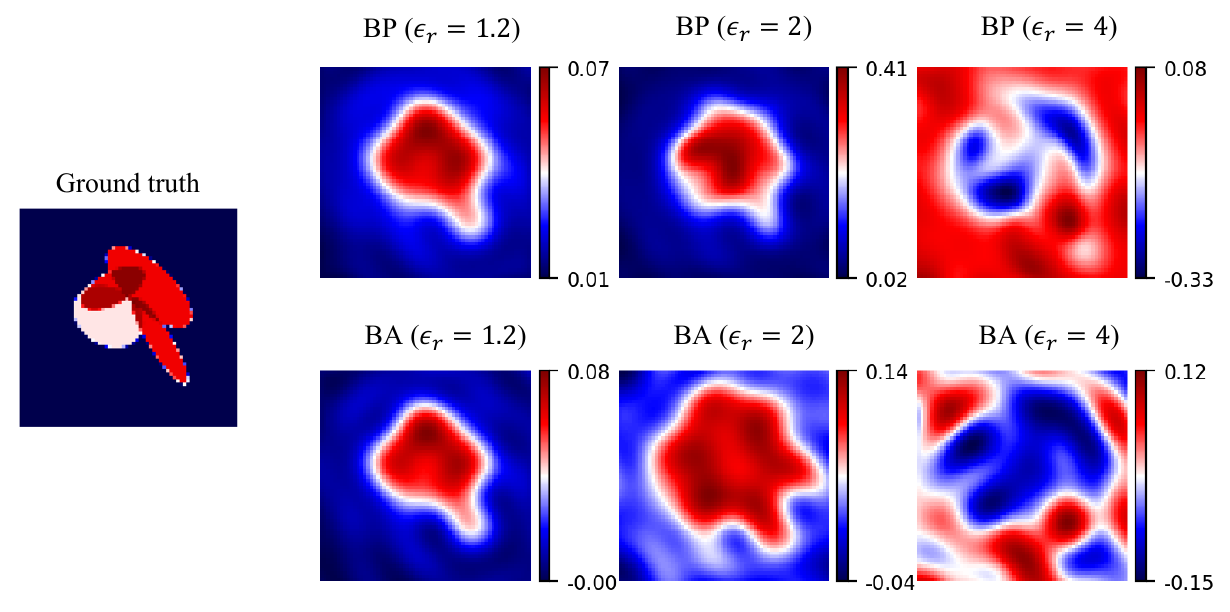}
    \caption{Performance analysis of the Back-Propagation (BP) and Born Approximation (BA) methods across objects with different maximum $\epsilon_r$ values. While both BP and BA reconstructions are visually meaningful for small $\epsilon_r$, their performance significantly deteriorates for objects with larger $\epsilon_r$.}
    \label{fig:Initial_Guesses}
\end{figure}

\begin{figure}
    \centering
    \includegraphics[width =0.35\textwidth]{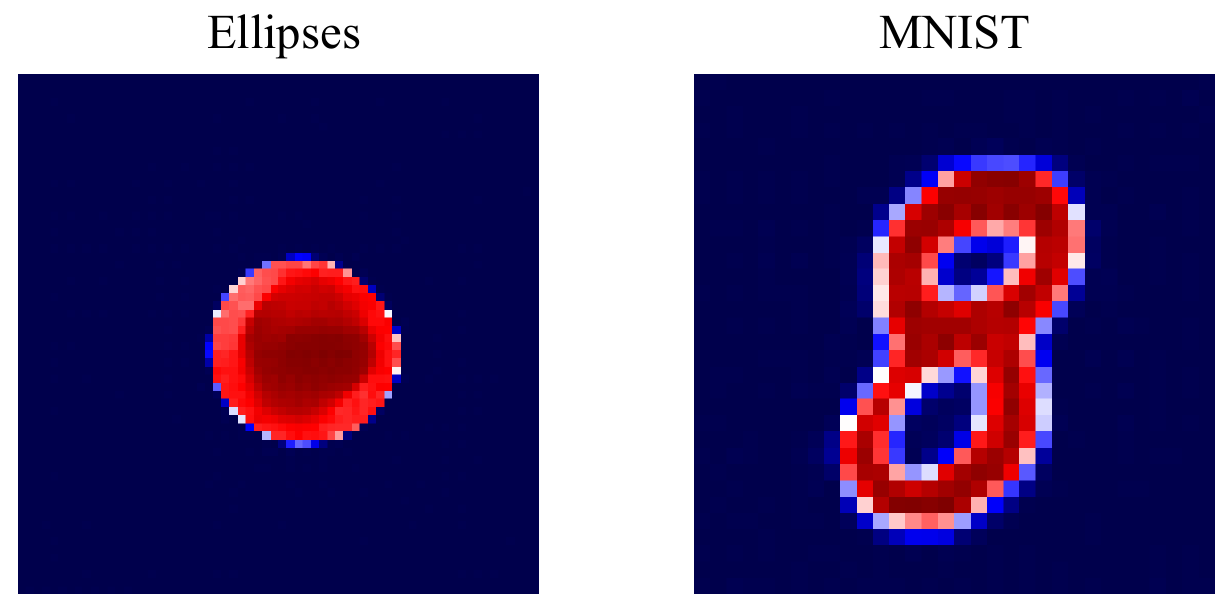}
    \caption{Illustration of the MOG initialization in the data space $f_\theta(z = 0)$ for ellipses and MNIST datasets.}
    \label{fig: MOG}
\end{figure}

\begin{figure*}
    \centering
    \includegraphics[width =0.85\textwidth]{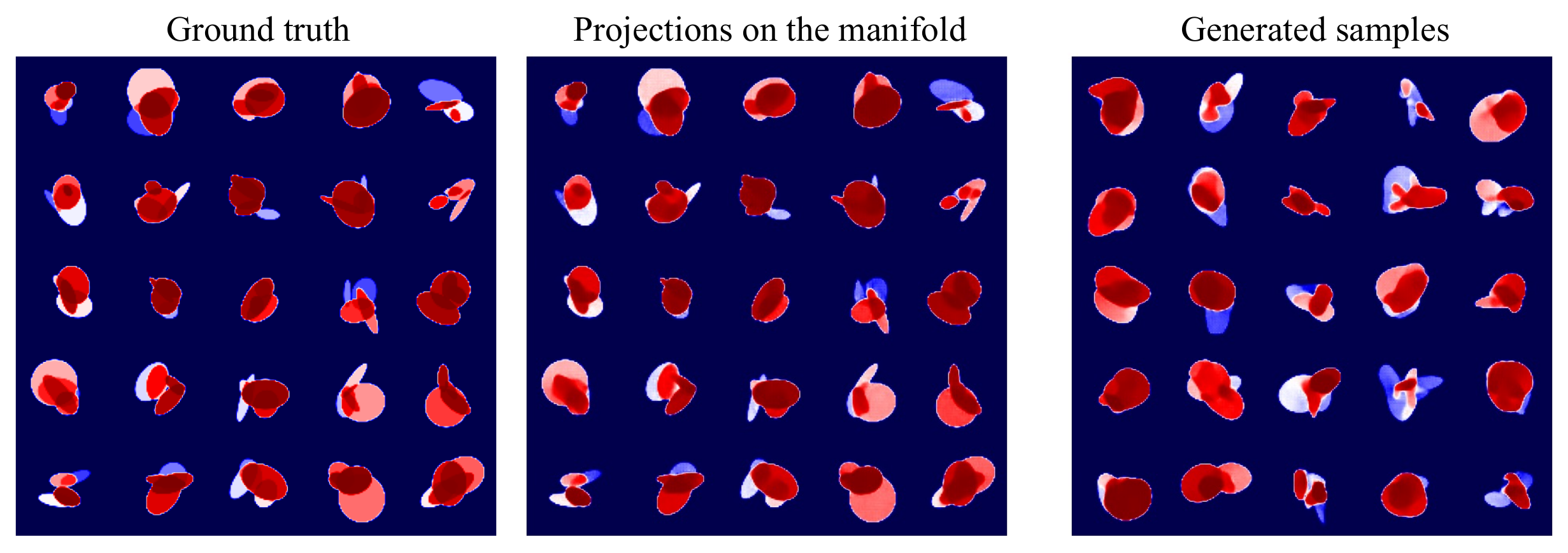}
    \caption{Performance evaluation of the trained injective flow on ellipses dataset; ground truth contrasts, their projections on the learned manifold and generated samples.}
    \label{fig: generated samples}
\end{figure*}

The choice of initial guess is important for inverse scattering solvers. A poor initialization may result in convergence to poor local minima due to nonlinearity. A good initial guess facilitates efficient convergence to good minima. The authors of \cite{chew1990reconstruction} used Born approximation as the initialization for the distorted Born iterative method (DBIM). A back-propagation (BP) solution was also used in \cite{van1997contrast, van1999extended} as an initial guess of the contrast source inversion (CSI) method. Figure~\ref{fig:Initial_Guesses} shows the ground truth, back-propagation (BP), and Born approximations (BA) for an object with different maximum $\epsilon_r$ values. While BP and BA may yield satisfactory results for objects with small permittivity, their performance sharply drops for large $\epsilon_r$ (especially numerically) which makes them a poor initialization for strong scatterers.

In order to circumvent this issue, we adopt a data-driven initialization
suggested in \cite{asim2020invertible}; mean of the Gaussian distribution (MOG) in the latent space which is set to 0. The MOG initialization $z = 0$, depicted in Figure~\ref{fig: Trumpets}, provides a fixed initialization with respect to the measurements (scattered fields); thereby being independent of the maximum contrast value and the problem configuration. This property leads to more robust convergence in both~\eqref{eq: LSO} and \eqref{eq: DSO} even for objects with large permittivity. \rev{While the DSO method can be initialized with both BP and MOG, the LSO should exclusively be initialized with MOG. This is due to the possibility of BP being significantly distant from the range of the injective network, making inversion to the latent space infeasible.} In
section~\ref{sec: Experiments}, we will show that the MOG significantly improves the quality of the reconstructions compared to BP, especially for strong scatterers.

\rev{
\section{Posterior Modeling and Uncertainty Quantification}
\label{sec: posterior modeling}

Due to ill-posedness, there are an infinite number of contrasts that are consistent with the measurements within the noise level. These diverse solutions can lead to different scientific interpretations, highlighting the need to characterize their distribution. Relying on a single estimate, such as the MAP estimate obtained in the previous section, fails to reflect the inevitable uncertainty and pinpoint features recovered only with low confidence. To address this drawback of point estimates, we adopt a Bayesian perspective. Rather than solely computing the MAP estimate, we approximate the \textit{full} posterior distribution $p_{X|Y}$ introduced in \eqref{eq: Bayes}. By doing so, we are able to generate many posterior samples which explore plausible permittivity patterns.

The computation of the posterior distribution, as stated in \eqref{eq: Bayes}, involves the integral $\int_x p_{X, Y}(x,y) dx$ which is intractable for high-dimensional imaging problems. Variational inference~\cite{hinton1993keeping,graves2011practical} is a promising framework that approximates the posterior distribution $p_{X|Y}(x|y)$ by defining a class of distributions $q_X(x;\psi)$ parameterized by $\psi$. The goal is to find the optimal $\psi$ that ensures the ``closeness'' between $q_X(x;\psi)$ and $p_{X|Y}(x|y)$ for a given $y$. Examples of such approximators include Gaussian mixture models and distributions induced by deep generative models.

In variational inference, a commonly used measure of fit is the Kullback--Leibler (KL) distance,
\begin{align*}
    \text{KL}(q \| p) 
    &= \int_\mathcal{X} q(x) \log\left(\dfrac{q(x)}{p(x)}\right) dx\\
    &= \E_{x \sim q}[ \log q(x) -  \log p(x)].
\end{align*}
We optimize $\psi$ to minimize the KL distance between $q_X(x;\psi)$ and $p_{X|Y}(x|y)$ for a given $y$,
\begin{align}
    \psi^* = \argmin_\psi \text{KL}(q_X(x;\psi) \| p_{X|Y}(x|y)).
    \label{eq: VI in data space}
\end{align}
Sun et al.~\cite{sun2021deep} parameterized $q_X(x;\psi)$ using an untrained normalizing flow through~\eqref{eq: density estimation} and directly performed the optimization over the network's weights.

We propose to leverage our pre-trained injective flow $f_\theta$ as a prior to approximate the posterior distribution. Our approach relies on the following principle: when we apply an injective mapping to the distributions $Q$ and $P$, resulting in new distributions $Q^\prime$ and $P^\prime$, respectively, the KL distance between $Q^\prime$ and $P^\prime$ remains the same as the KL distance between $Q$ and $P$ (for the formal theorem and the proof, refer to Section~\ref{sec: proof} in the appendix). This property of injective mappings motivates us to approximate the posterior distribution in the \textit{latent} space instead of the data space. Consequently, we minimize the KL distance between $q_Z(z,\psi)$ and $p_{Z|Y}(z|y)$ as follows,
\begin{align}
    \psi^* = & \argmin_\psi \text{KL}(q_Z(z,\psi) \| p_{Z|Y}(z|y)) \nonumber \\
     = & \argmax_\psi \mathbb{E}_{z \sim q_Z} \big[\log p_{Y|Z} (y|z) \big] - \text{KL}(q_Z \| p_Z) \nonumber \\
     = & \argmin_\psi \mathbb{E}_{z \sim q_Z} \big[\| y - A(f(z)) \|_2^2 \big] + \beta \big(\text{KL}(q_Z \| p_Z) \big),
    \label{eq: VI}
\end{align}
where $p_Z = \calN(0,I)$ represents the prior distribution introduced in~\eqref{eq: density estimation}. We consider $\beta$ as a
hyperparameter to control the diversity of the posterior samples as its value depends on the unknown noise power.

Now we must select our posterior approximator $q_Z(z,\psi)$. While previous works \cite{whang2021composing, kothari2021trumpets} used an additional normalizing flow to model $q_Z(z,\psi)$, we use a Gaussian distribution for simplicity and computational efficiency.
Specifically, we define $q_Z(z,\psi) = \calN(z;\mu_q, \text{diag}(\sigma_q))$, where $\psi = (\mu_q, \sigma_q)$ represents our variational parameters. This Gaussian parameterization of $q_Z(z,\psi)$ simplifies the KL term in \eqref{eq: VI} since there exists a closed-form expression for the KL distance between two Gaussian distributions,
\begin{align}
    \text{KL}(q_Z \| p_Z) = \dfrac{1}{2} \sum_{i = 1}^d \sigma_{q}(i)^2 + \mu_q(i)^2 -1 -2\log{\sigma_{q}(i)},
    \label{eq: KL between two Gaussians}
\end{align}
where $\mu_q(i)$ and $\sigma_q(i)$ denote the $i$th element of $\mu_q \in \R^d$ and $\sigma_q \in \R^d$, respectively.
Furthermore, since we have already obtained the MAP estimate in the latent space through \eqref{eq: LSO}, we set $\mu_q = z_{\text{MAP}}$ and only optimize $\sigma_q$.

We cannot directly optimize \eqref{eq: VI} using gradient-based methods since optimization variables are inside the expectation. We thus use the reparameterization trick~\cite{kingma2013auto,rezende2014stochastic}, letting $z = z_{\text{MAP}} + \sigma_q \odot t$, where $t \sim \calN(0, I)$ and $\odot$ denotes  the element-wise multiplication. By substituting \eqref{eq: KL between two Gaussians} into \eqref{eq: VI} and incorporating the above reparameterization,
\begin{align}
    \sigma_q^* = & \argmin_{\sigma_q}\mathbb{E}_{t \sim \calN(0,I)} \bigg[\| y - A(f(z_\text{MAP} + \sigma_q \odot t)) \|_2^2 \bigg] \nonumber \\
    & + \beta \sum_{i = 1}^d \bigg(\sigma_{q}(i)^2 -2\log{\sigma_{q}(i)} \bigg).
    \label{eq: Laplace approx true mean}
\end{align}
To evaluate the expectation, we compute the average over $K$ iid samples drawn from the standard normal distribution,
\begin{align}
    \sigma_q^* \approx & \argmin_{\sigma_q} \sum_{k = 1}^K \bigg(\| y - A(f(z_\text{MAP} + \sigma_q \odot t_k)) \|_2^2 \bigg) \nonumber \\
    & + \beta \sum_{i = 1}^d \bigg(\sigma_{q}(i)^2 -2\log{\sigma_{q}(i)} \bigg).
    \label{eq: Laplace approx}
\end{align}
Once we obtain the optimal $\sigma_q^*$, we can generate posterior samples $x_{\text{post}} = f(z_{\text{MAP}} + \sigma_q^* \odot t)$ where $t \sim \calN(0,I)$. Additionally, we can evaluate the empirical minimum mean-squared error (MMSE) estimate and the associated uncertainty by calculating the pixel-wise average and standard deviation over multiple posterior samples.}


\section{Computational Experiments}
\label{sec: Experiments}

\begin{figure*}
    \centering
    \begin{subfigure}{\textwidth}
    \centering
    \includegraphics[width = 0.99 \textwidth]{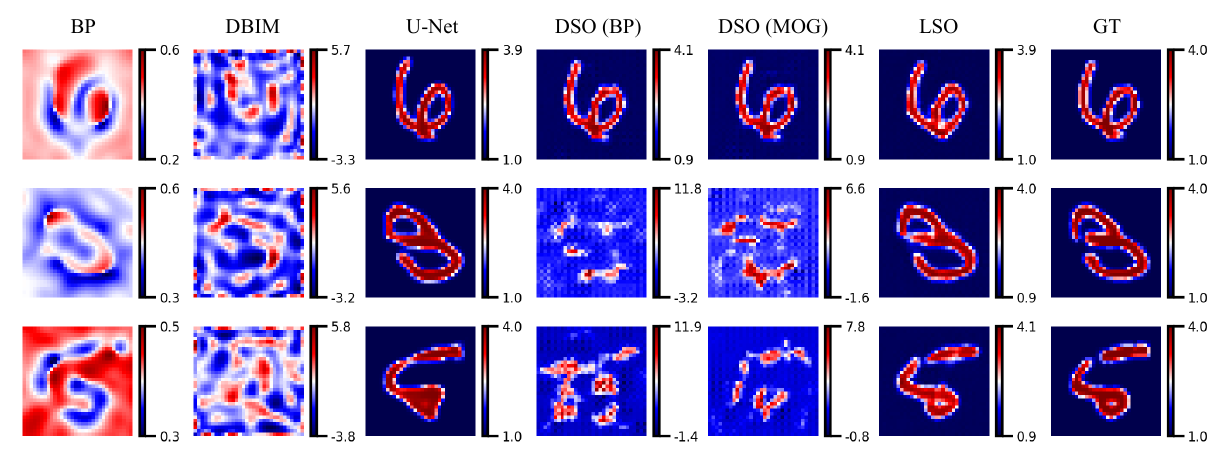}
    \caption{MNIST in resolution $32 \times 32$}
    \label{fig: mnist results}
    \end{subfigure}
    \begin{subfigure}{\textwidth}
    \centering
    \includegraphics[width = 0.99 \textwidth]{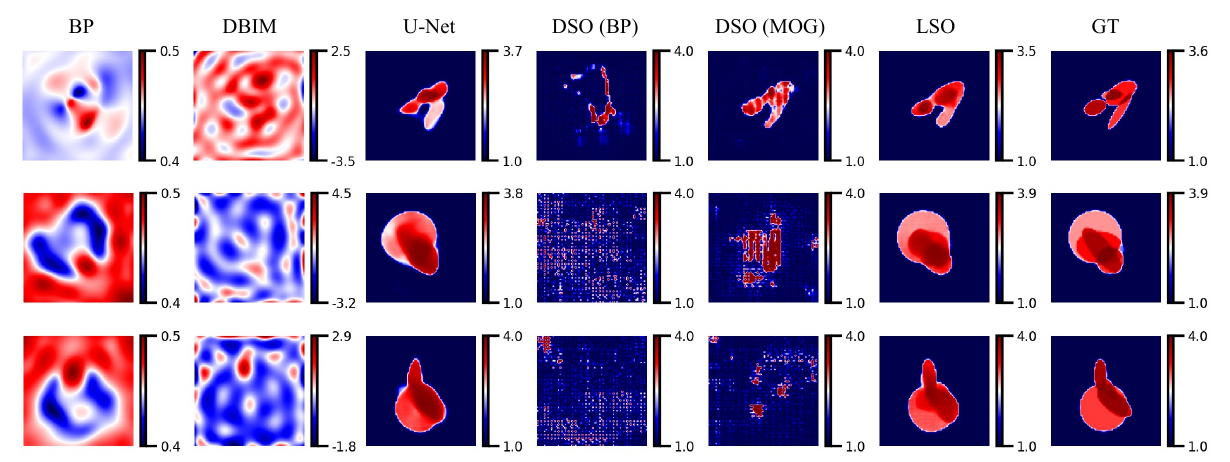}
    \caption{Ellipses in resolution $64 \times 64$}
    \label{fig: ellipses results}
    \end{subfigure}
    \caption{Performance comparison of different methods for objects with maximum $\epsilon_r = 4$.}
    \label{fig: MAP results}
\end{figure*}

We assess the performance of the proposed methods for MAP estimation and posterior modeling on synthetic and experimental data. We train the model on two synthetic large-scale datasets: 1) MNIST~\cite{lecun1998gradient} with 60000 training samples in the resolution $N = 32$, and 2) a more challenging dataset we generated comprising 60000 training samples with resolution $N = 64$ of overlapping ellipses used in~\cite{khorashadizadeh2022conditional}. Figure~\ref{fig: generated samples} shows example test contrasts, their projections on the learned manifold, and the samples generated by the injective network, verifying the ability of the model to produce outputs of good quality. For additional details about the network architecture and training, please refer to Section~\ref{sec: architecture and training} in the appendix. 

\subsection{Synthetic Data}
\rev{In experiments with synthetic data, the task is to reconstruct the test samples from MNIST and ellipses datasets that have not been ``seen'' by the injective network during training.}
We use $N_i = 12$ incident plane waves and $N_r = 12$ receivers, uniformly distributed on a circle with radius $R = 20$ cm around the object with maximum permittivity $\epsilon_r$ and dimension $D = 20$ cm. The working frequency is 3 GHz and we added 30 dB noise to the measurements of the scattered fields.

\paragraph{\textbf{MAP estimation}}
We conduct a comprehensive evaluation of the DSO and LSO methods. We consider the MOG and BP initializations for DSO while only using the MOG initialization for LSO. We compare the performance of our proposed methods with a traditional iterative method, DBIM~\cite{chew1990reconstruction}. While our approach is unsupervised so that the scattered fields are not used during training, we also compare its performance with a supervised learning method, the U-Net~\cite{ronneberger2015u}, which has enjoyed tremendous empirical success in a variety of imaging inverse problems including inverse scattering~\cite{wei2018deep}. The U-Net takes the BP image as input and regresses the corresponding permittivity.

\rev{We have fully implemented the forward operator in Tensorflow~\cite{abadi2016tensorflow}, enabling efficient GPU utilization for parallel reconstruction of multiple samples. Moreover, it allows us to use a variety of optimizers provided in Tensorflow including Adam~\cite{kingma2014adam} and L-BFGS~\cite{nocedal1999numerical}. In these experiments, we optimize~\eqref{eq: LSO} and~\eqref{eq: DSO} using the Adam optimizer with a learning rate of 0.05 for 300 iterations as it leads to more accurate reconstructions compared to L-BFGS.} We set $\lambda = 0.01$ for BP and $\lambda = 0$ for MOG. For the MOG initialization, we begin from high-likelihood regions (mean of the Gaussian), viewed as a hidden regularizer and we thus set $\lambda = 0$. Figure~\ref{fig: MOG} illustrates the MOG initializations for ellipses and MNIST datasets.

\begin{table}
\renewcommand{\arraystretch}{1.3}
\caption{Performance of different methods for solving inverse scattering ($\epsilon_r = 4$) averaged over 5 test samples.}
\centering
\begin{tabular}{ c| c c| c c  }
  & \multicolumn{2}{c|}{\textbf{PSNR}} & \multicolumn{2}{c}{\textbf{SSIM}} \\
 \hline
  & \textbf{MNIST} & \textbf{Ellipses} &  \textbf{MNIST} & \textbf{Ellipses} \\
\hline
  BP & 7.75 & 7.00 & 0.01 & 0.01 \\
  DBIM~\cite{chew1990reconstruction} & 5.77 & 4.67 & 0.01 & 0.01  \\
  U-Net~\cite{ronneberger2015u} & 24.26 & \textbf{21.94} & \textbf{0.90} &   0.82 \\
  DSO (BP) & 8.73 & 7.89  & 0.16 &  0.16 \\
  DSO (MOG) & 17.47 & 14.56 & 0.61 & 0.44 \\
  LSO (MOG) & \textbf{25.22} & 20.50 & 0.89 & \textbf{0.85} \\
\end{tabular}
\label{tab: Main results}
\end{table}

Figure~\ref{fig: MAP results} shows the performance of various methods for $\epsilon_r = 4$ using 5 test samples from MNIST and ellipses datasets. While DBIM falls short in this challenging task with a high contrast and 30 dB noise, DSO and LSO exhibit much better reconstructions. Moreover, the MOG initialization, as expected, yields superior reconstructions compared to BP. Notably, LSO outperforms DSO, demonstrating the advantages of running optimization in the latent space as discussed in Section~\ref{sec: MAP}. Despite not utilizing scattered fields during the training phase, LSO produces reconstructions of comparable or even superior quality to the supervised method U-Net. Table~\ref{tab: Main results} lists the numerical results in PSNR and SSIM averaged over 5 test samples.

As discussed in Section~\ref{sec: MAP}, the maximum $\epsilon_r$ of the object plays a significant role in the performance of inverse scattering solvers. Figure~\ref{fig: SNR for different er} shows the performance of various methods across different maximum $\epsilon_r$ values on MNIST. This analysis shows that LSO, combined with the MOG initialization, remains effective even for objects with high $\epsilon_r$, which highlights the significance of data-driven initialization and optimization in the latent space. 

\rev{Regarding the computational efficiency, we used a single Tesla V100 GPU for training and solving the inverse scattering problem where each iteration of LSO (or DSO) takes 0.08 seconds at the resolution of $N = 32$ and 0.25 seconds at the resolution of $N = 64$. Although good estimates can be obtained with much fewer iterations, we empirically determined that 300 iterations ensure good convergence.}

\begin{figure}
    \centering
    \includegraphics[width = 0.49\textwidth]{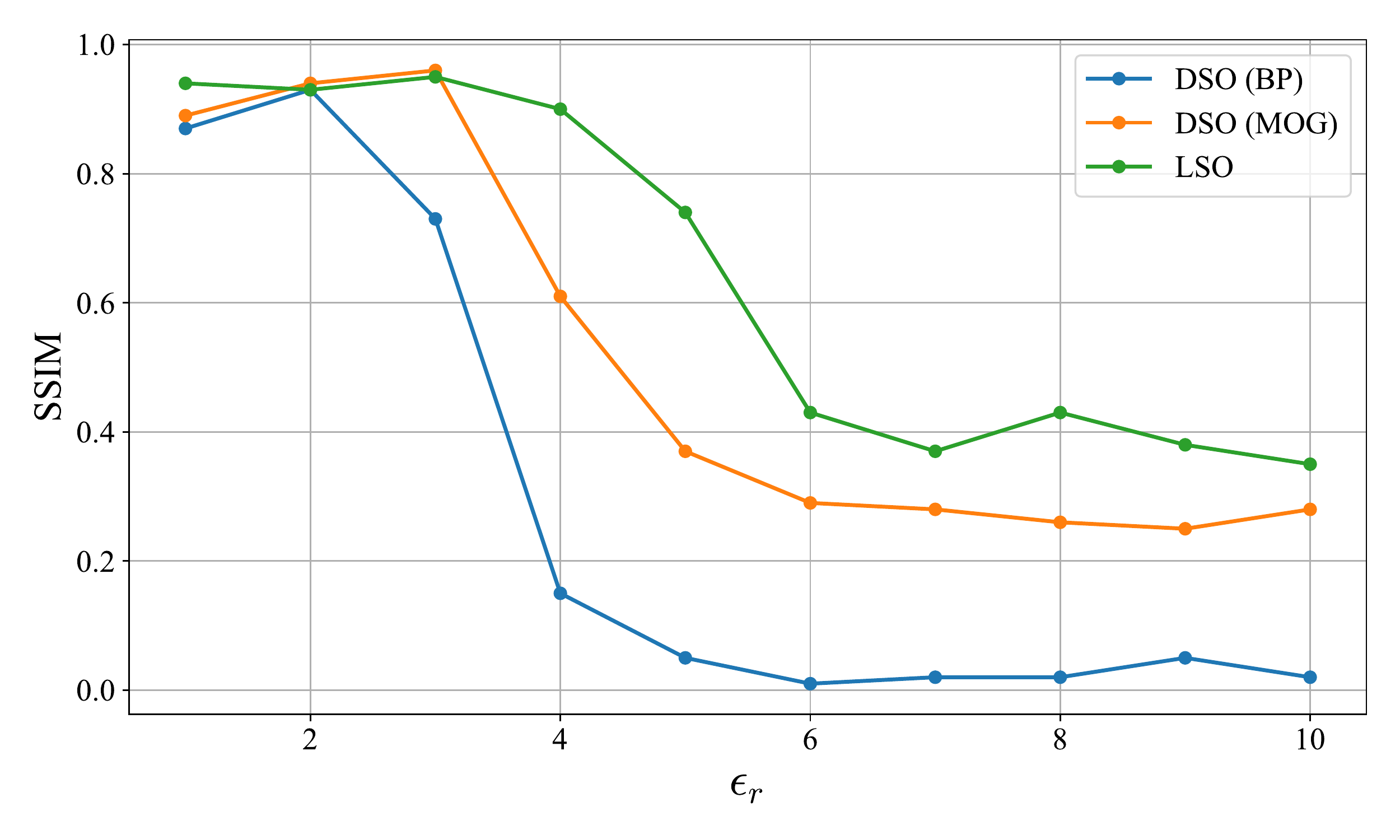}
    \caption{Performance of various methods across objects with different maximum $\epsilon_r$ values on the MNIST dataset.}
    \label{fig: SNR for different er}
\end{figure}

\rev{
\paragraph{\textbf{Posterior Sampling and UQ}}
As explained in Section~\ref{sec: posterior modeling}, we approximate the posterior distribution of contrast as a pushforward of a Gaussian around the MAP estimate in the latent space; the covariance is chosen to obtain the best variational approximation of the posterior in the sense of the KL divergence. We use the MAP estimate obtained from the LSO method in the previous section and optimize \eqref{eq: Laplace approx} using the Adam optimizer with a learning rate of $0.01$. The initial value for $\sigma_q$ is set as an all-one vector, and we use $K = 25$ random samples drawn from the standard Gaussian in each iteration. To compute the MMSE estimate and UQ, we calculate the pixel-wise average and standard deviation over 25 posterior samples. Figure~\ref{fig: posterior samples} showcases 4 posterior samples along with UQ and MMSE estimates for $\beta = 0.01$ and $\beta = 0.05$. As expected, larger $\beta$ values lead to more diverse posterior samples.  The UQ map identifies regions with higher uncertainty visually represented in red. This information is highly valuable for conducting a more thorough and informed analysis. Finally, the MAP estimate is sharper than the MMSE as expected.}

\begin{figure*}
    \centering
    \includegraphics[width = 0.99\textwidth]{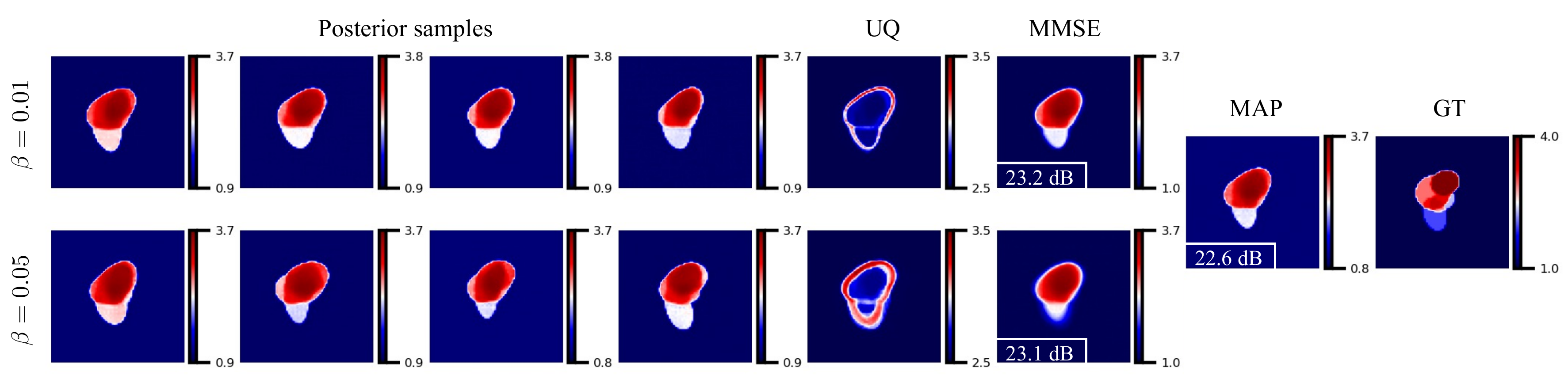}
    \caption{Posterior samples, UQ, MMSE, and MAP estimates for an object with $\epsilon_r = 4$ for $\beta = 0.01$ and $\beta = 0.05$; as expected, larger $\beta$ values lead to more diverse posterior samples.}
    \label{fig: posterior samples}
\end{figure*}

\rev{
\paragraph{\textbf{Generalization}}
In this section, we evaluate the generalization performance of the proposed method under out-of-distribution changes in the permittivity patterns. We train injective flows exclusively on MNIST digits 0-5 and use the remaining digits for testing. The LSO solver is configured with the same setup as in the previous section. Figure~\ref{fig: posterior ood} shows the posterior samples, UQ, MMSE, and MAP estimates for two test samples of digits 6 and 8 with $\beta = 0.05$. This experiment clearly shows the effectiveness of the proposed method in handling out-of-distribution data. We should point out that there exists a trade-off between regularization power and generalization performance, governed by the dimension of the latent space. Larger latent space dimensions yield better generalization but less effective regularization. This has also been observed in regular normalizing flows, where matching dimensions in the latent and data space result in excellent generalization over out-of-distribution data but less effective regularization \cite{asim2020invertible,pmlr-v202-liu23au}.

\begin{figure*}
    \centering
    \includegraphics[width = 0.99\textwidth]{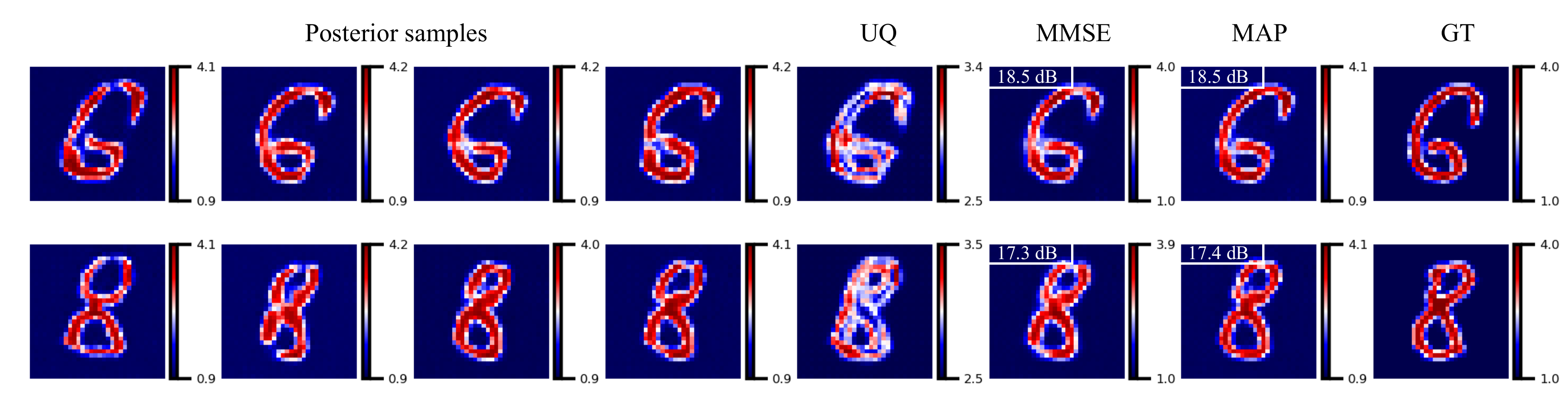}
    \caption{Reconstructions and UQ for out-of-distribution samples with $\epsilon_r = 4$. Despite being trained solely on MNIST digits 0-5, the proposed method exhibits excellent generalization by accurately reconstructing digits 6-9.}
    \label{fig: posterior ood}
\end{figure*}

\subsection{Experimental Data}
We finally evaluate our proposed model on FoamDielExt and FoamTwinDiel: real experimental data for two phantoms provided by the Institute Fresnel in Marseille, France~\cite{geffrin2005free}. In these experiments, there are $N_i = 8$ transmitters and 241 receivers located on a circle with radius $R = 1.67$ m. Out of those, we only use $N_r = 20$ receivers to make the inversion more challenging. Additional details about the setup are discussed in~\cite{geffrin2005free}. As shown in Figure~\ref{fig: fresnel gt}, FoamDielExt and FoamTwinDiel consist of dielectric cylinders in a vacuum background. We use the measurements at the working frequency of 3 GHz, and the side length of the investigation domain is $D = 20$ cm. 

We use two pre-trained injective flows on the ellipses dataset for resolutions $N = 32$ and $N = 64$. The inverse scattering problem is solved using \eqref{eq: LSO} for MAP estimation and \eqref{eq: Laplace approx} for posterior modeling. We added the total-variation (TV) regularization term to \eqref{eq: LSO} and \eqref{eq: Laplace approx} to further improve the quality of the reconstruction. The TV-norm multiplier is 0.1 and 0.08 for resolutions $N = 32$ and $N = 64$, respectively.
Figure~\ref{fig: fresnel results} shows posterior samples, UQ, MMSE, and MAP estimates.
Despite the idealized forward operator 
 and the substantial dissimilarity between the ground truth (two or three circles) and the training data (combinations of four ellipses with random positions and contrasts), the proposed framework produces satisfactory reconstructions. This experiment illustrates the robustness of the proposed method to noise and variations in experimental configuration. It also showcases the importance of posterior modeling: while the MAP and MMSE estimates in Figure~\ref{fig: fresnel_results_32} wrongly reconstruct the larger circle as compared to the ground truth, the uncertainty maps clearly signal that this part of the recovered contrast is not reliable.
}

\begin{figure}
\centering
\begin{subfigure}{0.25\textwidth}
  \centering
 \includegraphics[width=0.9\textwidth]{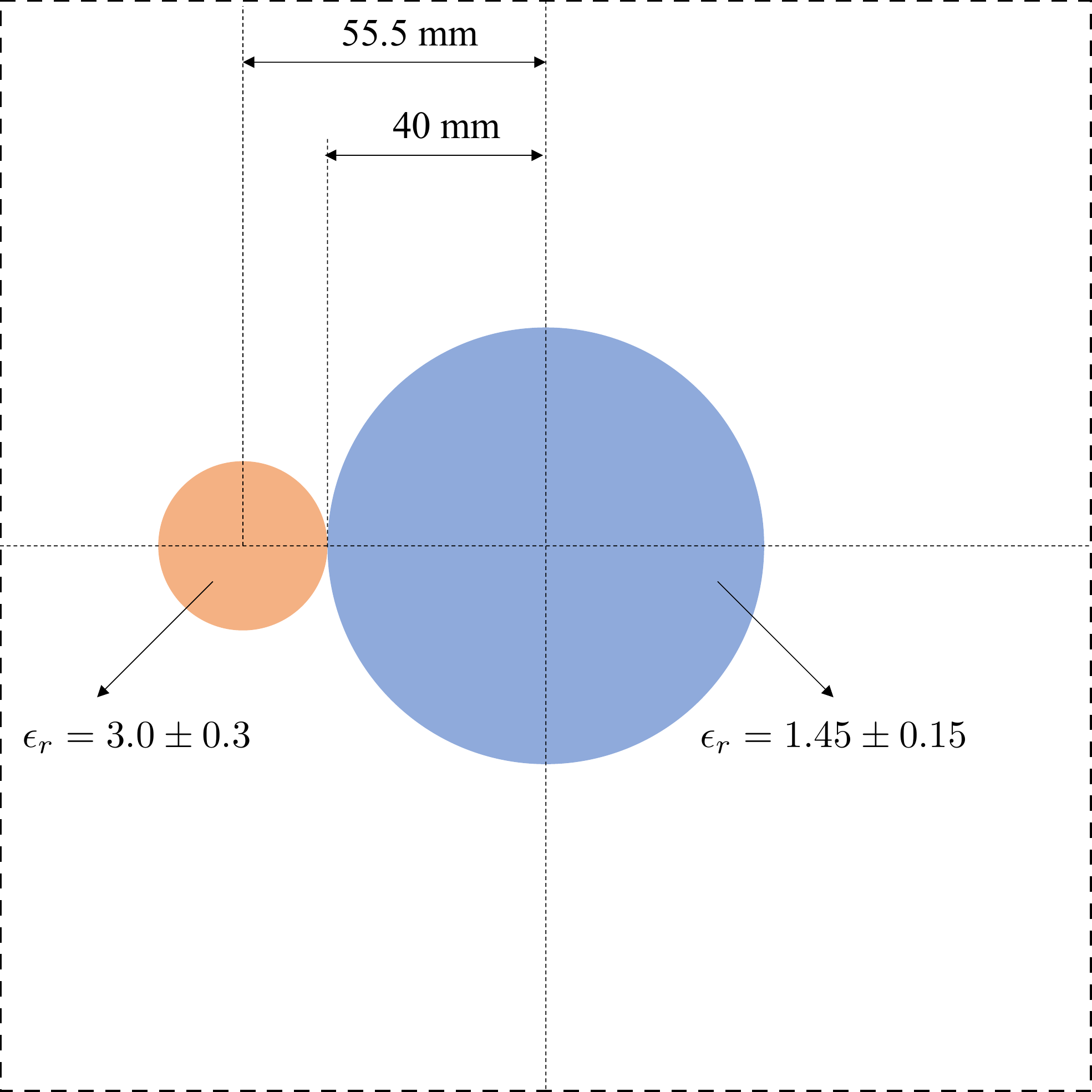}
\caption{FoamDielExt}
\end{subfigure}%
\begin{subfigure}{0.25\textwidth}
\centering
\includegraphics[width=0.9\textwidth]{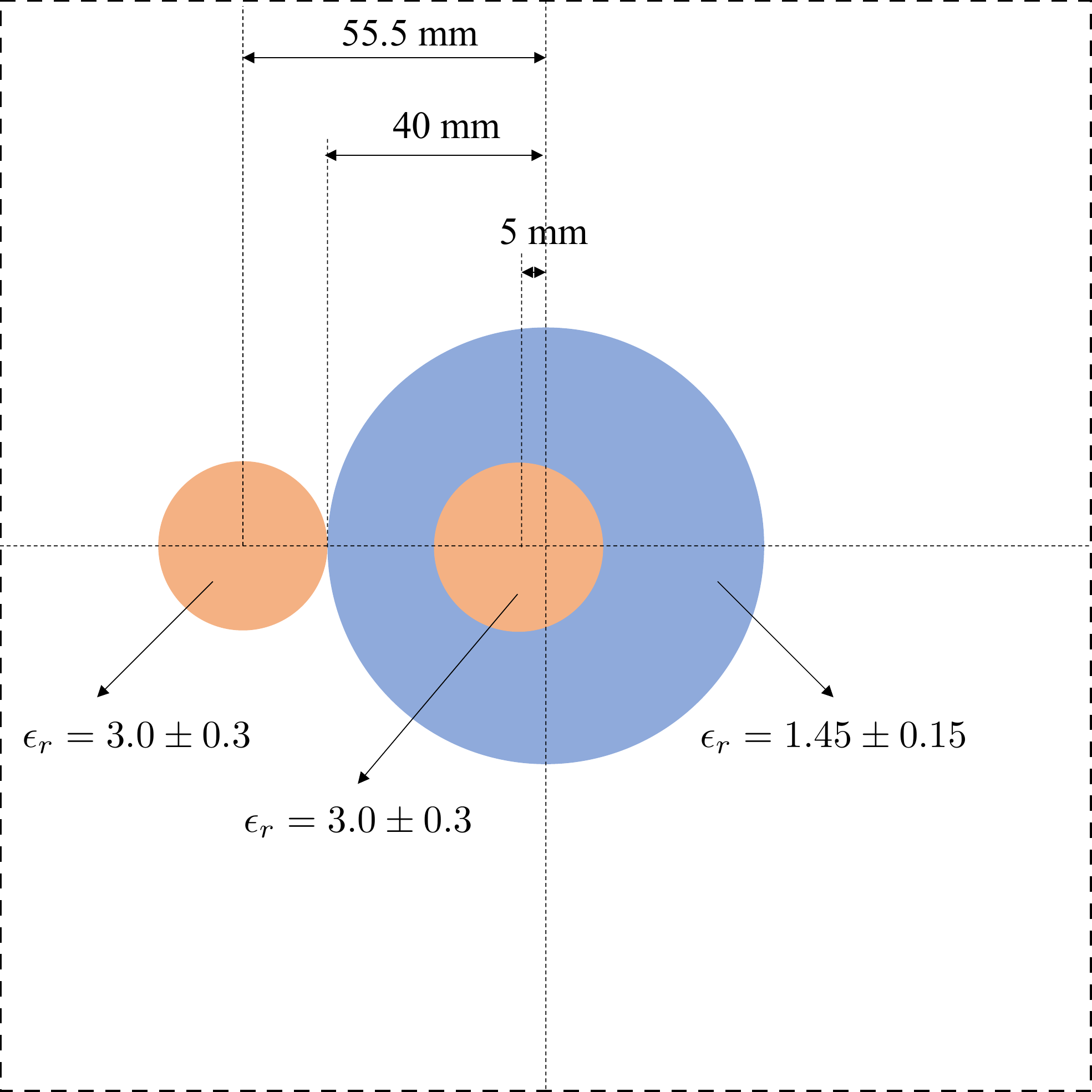}
\caption{FoamTwinDiel}
\end{subfigure}
\caption{Experimental Fresnel data~\cite{geffrin2005free}
}
\label{fig: fresnel gt}
\end{figure}

\begin{figure*}
\centering
\begin{subfigure}{\textwidth}
  \centering
 \includegraphics[width=\textwidth]{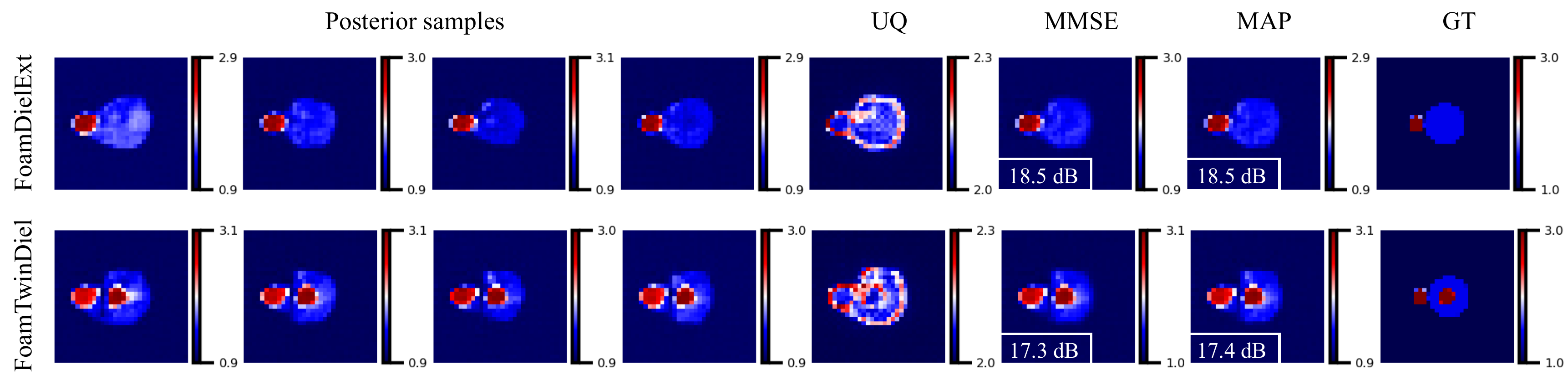}
\caption{Resolution $32 \times 32$}
\label{fig: fresnel_results_32}
\end{subfigure}
\begin{subfigure}{\textwidth}
\centering
\includegraphics[width=\textwidth]{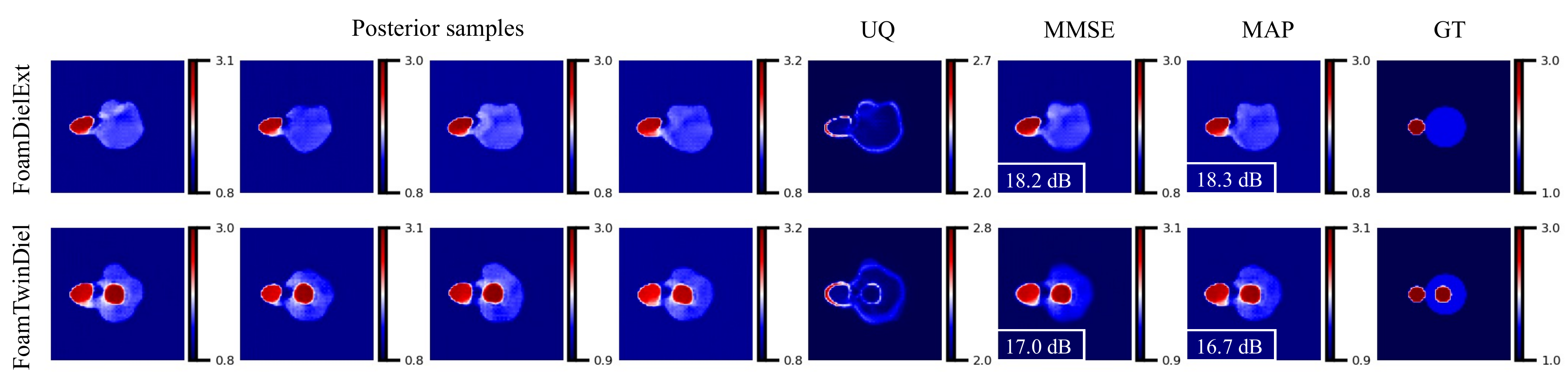}
\caption{Resolution $64 \times 64$}
\label{fig: fresnel_results_64}
\end{subfigure}
\caption{Posterior samples, UQ, MMSE, and MAP estimates for experimental Fresnel data. The uncertainty maps clearly signify the importance of posterior modeling by assigning higher uncertainty to wrongly reconstructed areas (red regions).
}
\label{fig: fresnel results}
\end{figure*}

\section{Limitations and Conclusions}
\label{sec: conclusion}
We proposed a data-driven framework for inverse scattering using an injective prior. The proposed method fully exploits the physics of wave scattering while benefiting from a data-driven initialization resulting in a powerful solver even for objects with a large contrast. The invertible generator admits optimization in both latent and data space and uses either a data-driven initialization or a back-projection. We showed that optimization in the latent space and with the latent Gaussian center as the initial guess significantly outperforms traditional iterative methods and even gives reconstructions comparable to a strong supervised method, the U-Net. 

\noindent \emph{Limitations and Future Works:}

The proposed framework has several key limitations. It requires running an iterative method at test time, which is slow and impractical for real-time applications. \rev{Moreover, iterative methods can converge to local minima even with clever initialization.} To speed up convergence, one may consider a more accurate initial guess by exploiting physics in the data-driven initialization via a combination of traditional back-projection (like BP) and data-driven initializations (like MOG). \rev{Furthermore, while the L-BFGS optimizer didn't improve the convergence rate in our experiments, other Newton's family optimizers may improve the convergence rate as shown in~\cite{guo2022nonlinear}.
Additionally, forcing the reconstruction to be within the range of an injective flow can introduce undesired bias and artifacts in certain applications.} Recently, Hussein et al.~\cite{hussein2020image} optimized the generator weights with a small rate after finding the optimal latent code in~\eqref{eq: LSO} to further improve the reconstructions; this idea might be adapted to our framework. We leave addressing these limitations for future work.

\input{appendix}

\newpage
\bibliographystyle{IEEEtran}
\bibliography{Paper}

\newpage

 




\vfill
\end{document}

%% file: package.tex
\usepackage{url}

\usepackage{array}
\usepackage{textcomp}
\usepackage{stfloats}
\usepackage{verbatim}
\usepackage{cite}
\usepackage{caption}
\usepackage{subcaption}
\usepackage{algorithmic}
\usepackage{algorithm}
\usepackage{graphicx}
\usepackage{amsmath,amsfonts}
\usepackage{amssymb}
\usepackage{booktabs}
\usepackage[pagebackref=true,breaklinks=true,colorlinks,bookmarks=false]{hyperref}
\usepackage{orcidlink}
\usepackage{amsthm}

\usepackage[capitalize]{cleveref}
\crefname{section}{Sec.}{Secs.}
\Crefname{section}{Section}{Sections}
\Crefname{table}{Table}{Tables}
\crefname{table}{Tab.}{Tabs.}

\usepackage{epsfig}

\usepackage{microtype}
\usepackage{subcaption}
\usepackage{xcolor}

\usepackage[utf8]{inputenc}
\usepackage{multirow}
\usepackage{outlines}

\newcommand{\lmse}{\calL_\text{MSE}}
\newcommand{\lml}{\calL_\text{ML}}

\newcommand{\calL}{\mathcal{L}}
\newcommand{\calN}{\mathcal{N}}

%% file: math_commands.tex
\usepackage{amsmath,amsfonts,bm}

\newtheorem*{lemma}{Lemma}

\def\1{\bm{1}}

\DeclareMathAlphabet{\mathsfit}{\encodingdefault}{\sfdefault}{m}{sl}
\SetMathAlphabet{\mathsfit}{bold}{\encodingdefault}{\sfdefault}{bx}{n}

\newcommand{\E}{\mathbb{E}}

\newcommand{\R}{\mathbb{R}}

\DeclareMathOperator*{\argmax}{arg\,max}
\DeclareMathOperator*{\argmin}{arg\,min}

%% file: appendix.tex
\appendix
\subsection{Network Architecture and Training Details} 
\label{sec: architecture and training}
\rev{The injective subnetwork $g_{\gamma}$ is composed of 6 injective revnet blocks described in Section~\ref{sec: Trumpets}, each increasing the dimension by a factor of 2. To enhance the expressiveness of the model, we insert 36 bijective revnet blocks between them. We choose a latent space of dimension 64 which provides a compression rate of $98.5\%$ for resolution $N = 64$ and $93.7\%$ for resolution $N = 32$. The bijective subnetwork $h_\eta$  is constructed using 20 bijective revnet blocks.}

\rev{We normalize the training data between 0 and 1 before training the model. We then multiply the output of the trained network by the maximum contrast of the dataset before using it as the generative prior.} We train the injective subnetwork $g_\gamma$ for 150 epochs to ensure the training samples (contrast patterns) align with the generator's range. Following this, we train the bijective subnetwork $h_\eta$ for 150 epochs to maximize the likelihood of the training samples in the intermediate space.

\subsection{The Invariance of KL Distance under Injective Mappings}
\label{sec: proof}
\begin{lemma} We assume probability distributions $q_Z$ and $p_Z $ have the same support. We let $q_X = f_\# q_Z$ and $p_X = f_\# p_Z$ where $f_\# p$ denotes the
pushforward of $p$ via mapping $f$, i.e., for every $x$ from $p$, $f(x)$ is a sample from $f_\# p$~\footnote{For simplicity we lightly abuse notation by identifying a probability measure and its density.}. If $f$ is injective then it holds,
\begin{align}
    \text{KL}(q_X \| p_X) = \text{KL}(q_Z \| p_Z)
\end{align}
\end{lemma}

\begin{proof}
The change of variable for the injective mapping $f$ yields~\cite{boothby1986introduction},
\begin{align}
\label{eq:posterior_distribution}
\log p_{X}(x) = \log p_Z(z)
- \dfrac{1}{2}\log |\det [J_{f}(z)^T
 J_{f}(z)]|,
\end{align}
where $z = f^{\dagger}(x)$ and is valid for $x \in \text{Range}(f)$. Now, we can compute the KL distance in the data space as follows,
\begin{align}
    \text{KL}(q_X \| p_X) = & \E_{x \sim q_X}[ \log q_X(x) -  \log p_X(x)] \nonumber \\ 
     = & \E_{x \sim q_X}[ \log q_Z(z)
 - \dfrac{1}{2}\log |\det [J_{f}(z)^T
 J_{f}(z)]| \nonumber \\
& - \log p_Z(z)  + \dfrac{1}{2}\log |\det [J_{f}(z)^T
 J_{f}(z)]|] \nonumber \\
 = & \E_{x \sim q_X}[ \log q_Z(z) -  \log p_Z(z)] \nonumber \\
= & \E_{z \sim q_Z}[ \log q_Z(z) -  \log p_Z(z)] \nonumber \\
= & \text{KL}(q_Z \| p_Z), \nonumber
\end{align}
which establishes the lemma.
\end{proof}